%% file: composite_kdd.tex
\documentclass{sig-alternate-05-2015}

\usepackage{IEEEtrantools}
\usepackage{verbatim}

\usepackage{flushend}

\usepackage{amsmath}
\usepackage{amssymb}

\usepackage{algorithmic}
\usepackage{mathrsfs}  
\usepackage{array}
\usepackage{caption}
\DeclareCaptionType{copyrightbox}
\usepackage{subcaption}

\DeclareCaptionFormat{subfig}{\figurename~#1#2#3}
\DeclareCaptionSubType*{figure}
\newtheorem{theorem}{Theorem}
\newtheorem{lemma}{Lemma}
\newtheorem{corollary}{Corollary}

\newcommand{\bcomment}[1]{}
\newcommand{\TODO}[1]{}

\begin{document}
\setcopyright{acmcopyright}


\conferenceinfo{KDD '16}{August 13--17, 2016, San Francisco, CA, USA}

\acmPrice{\$15.00}

%

\title{On the Simultaneous Preservation of Privacy and Community Structure in Anonymized Networks}

\numberofauthors{4}

\author{
\alignauthor
Daniel Cullina\\
\affaddr{Department of Electrical and Computer Engineering}\\
\affaddr{University of Illinois, Urbana, Illinois 61801}\\
\email{cullina@illinois.edu}
\alignauthor
Kushagra Singhal
\affaddr{Department of Electrical and Computer Engineering}\\
\affaddr{University of Illinois, Urbana, Illinois 61801}\\
\email{ksingha2@illinois.edu}
\alignauthor Negar Kiyavash
\affaddr{Department of Electrical and Computer Engineering}\\
\affaddr{University of Illinois, Urbana, Illinois 61801}\\
\email{kiyavash@illinois.edu}
\and
\alignauthor 
Prateek Mittal\\
\affaddr{Department of Electrical Engineering}\\
\affaddr{Princeton University, Princeton, New Jersey 08544}\\
\email{pmittal@princeton.edu}
}

\bcomment{

\IEEEauthorblockN{
Daniel Cullina\IEEEauthorrefmark{1}, 
Kushagra Singhal\IEEEauthorrefmark{1},
Negar Kiyavash\IEEEauthorrefmark{1},
and Prateek Mittal\IEEEauthorrefmark{2}}

\IEEEauthorblockA{\IEEEauthorrefmark{1}Department of Electrical and Computer Engineering\\
University of Illinois, Urbana, Illinois 61801\\
Email: \{cullina, ksingha2, kiyavash\}@illinois.edu}

\IEEEauthorblockA{\IEEEauthorrefmark{2}Department of Electrical Engineering\\
Princeton University, Princeton, New Jersey 08544\\
Email: pmittal@princeton.edu}
}

\maketitle
\begin{abstract}
\bcomment{The proliferation of online social networks has helped in generating large amounts of graph data which has immense value for data analytics. Network operators, like Facebook, often share this data with researchers or third party organizations, which helps both the entities generate revenues and improve their services. As this data is shared with third party organizations, the concern of user privacy becomes pertinent. Hence, it becomes essential to balance utility and privacy while releasing such data.}

We consider the problem of performing community detection on a network, while maintaining privacy, assuming that the adversary has access to an auxiliary correlated network. We ask the question ``Does there exist a regime where the network cannot be deanonymized perfectly, yet the community structure could be learned?." To answer this question, we derive information theoretic converses for the perfect deanonymization problem using the Stochastic Block Model and edge sub-sampling. We also provide an almost tight achievability result for perfect deanonymization.

We also evaluate the performance of percolation based deanonymization algorithm on Stochastic Block Model data-sets that satisfy the conditions of our converse. 
Although our converse applies to exact deanonymization, the algorithm fails drastically when the conditions of the converse are met. 
Additionally, we study the effect of edge sub-sampling on the community structure of a real world dataset. 
Results show that the dataset falls under the purview of the idea of this paper. 
There results suggest that it may be possible to prove stronger partial deanonymizability converses, which would enable better privacy guarantees.
\end{abstract}

\input{intro}

\input{related_work}

\input{model}

\input{theorems}

\input{communities}

\input{experiments}

\input{conclusion}

\bibliographystyle{abbrv}
\bibliography{ref}

\flushend

\end{document}

%% file: intro.tex
\section{Introduction}
\label{intro}

Data analytics is a rapidly growing field, aided by the availability of huge amounts of data and significant computing power. An enormous part of data generation is a result of the emergence of Online Social Networks (OSNs) such as Facebook, LinkedIn, Twitter etc. The user base of these networks spans in millions and is still growing. These companies and others perform data analytics for the purpose of increasing revenues, reducing customer service costs, better prediction and possibly prevention of attrition rates, getting feedback on and improving public opinion of their products/services. For instance, LinkedIn has been very successful in converting the data collected on their website into new data products, such as their \emph{People You May Know} feature. Network providers also create revenue by sharing data with other third parties who create value by performing analytics on the data. For example, due to homophily \cite{mcpherson2001birds}, OSNs are good microcosms to study efficient advertising strategies. With the prevalence of data analytics, concerns about user privacy are growing too and such concerns could hamper the former if not addressed adequately.

Preprocessing the data prior to its release, with the goal of minimizing the risk of sharing private information of the users, is crucial for addressing privacy concerns. \textit{Anonymization} is an essential step in the data preprocessing. Perhaps, still the most widely used technique is the naive practice of substituting the personal identifiers (e.g., name, IP address, etc) by random identifiers. More clever techniques such as k-anonymization \cite{hay2008resisting,liu2008towards} and differential privacy \cite{dwork2011differential,li2013membership} are also proposed to address the problem in suitable scenarios.

As fundamentally any anonymization technique involves modification of the data at some level, it could possibly deteriorate the utility of the data for the initial analytics tasks it was released for. This trade-off between privacy and data utility has been noted in the literature \cite{brickell2008cost,li2009tradeoff}, but a theoretical understanding of this trade-off is still missing.

\noindent{\bf Contributions.} 
We investigate the feasibility of performing data analytics, without compromising privacy of the users involved, in the following specific setting. 
Let $G_2$ denote a graph whose vertices are the identities of the users (e.g. names, email id, etc.) and its edges encode relationships (e.g. friendship, citations, professional relation etc.) among those users. 
Furthermore, assume that the vertices of $G_2$ are associated with some sensitive information (e.g. sexual orientation, personal preferences, hometown, relationship status, location history, etc.). 
A third party is interested in studying the relationship between the sensitive tags and the structural properties of $G_2$.
To preserve the privacy of the users, a sanitized version of $G_2$ would be released. 
Assume another graph $G_1$, correlated with $G_2$ and defined on the same vertex set, is available to the third party as {\em auxiliary} information.
In $G_1$, vertices are labeled with user identities, but no sensitive vertex tags are present. 
Given the availability of the public graph $G_1$, we ask the question: {\em Can we safely release a sanitized graph without compromising the privacy of the people involved?} 
The challenge lies in the requirement that the sanitized version of $G_2$ should allow for reliably performing analytics, but not allow the third party (referred to as the attacker) to learn the identity of the users (i.e., vertex labels) despite the availability of $G_1$.
In the rest of the paper, we limit our attention to a specific problem, the so-called \emph{community reconstruction}.
We selected this as a proxy for a much broader class of grass-analysis questions because it has structural features in common with many other problems and recent research work has established a detailed understanding of its properties.

Our main contributions are as follows.
We derive information theoretic converses for the anonymous exact community recovery problem for a large class of random graphs. 
That is, we provide a threshold in terms of the problem parameters, which if met, guarantees that no algorithm can deanonymize the graph. 
More specifically:
\begin{itemize}
\item We derive a nearly sharp threshold for exact deanony-mization problem for the class of SBM graphs.
\vspace{-5pt}
\item We characterize partial deanonymity of the system (i.e. the growth rate of number of vertices that cannot be deanonymized) as a function of the correlation between the {\em auxiliary} graph $G_1$ and the {\em sensitive} graph $G_2$ and the sparsity of these two graphs.
\vspace{-5pt}
\item We establish that there is a nonempty parameter space such that $G_2$ cannot be fully deanonymized but exact community recovery is feasible.
\vspace{-5pt}
\item We investigate methods of modifying $G_2$ to strengthen its anonymity while preserving community structure. 
\vspace{-5pt}
\item We study the behavior of the threshold identified by our converse as a function of growth rate of communities with regards to the size of the vertices, $n$. 
\end{itemize}

To the best of our knowledge, this is the first paper to offer a converse: a statement that under certain conditions, any deanonymization algorithm must fail. 
The previous work on the subject only provides {\em achievability} results for the problem, which
describe sufficient conditions on the model parameters under which deanonymization is possible. This is done by proving success of specific algorithms in deanonymizing the graph for a range of problem parameters or by providing simulation results on specific datasets \cite{narayanan2009anonymizing, nilizadeh2014community,pedarsani2011privacy,ji2014structural,ji2015your}. 
Instead, we seek {\em converses} that guarantee no algorithm is able to deanonymize the sanitized graph. As first steps to solving the problem, we study the converse such that no algorithm is able to deanonymize the network \emph{perfectly}.

The rest of the paper is organized as follows. In Section \ref{commreconstruct}, we describe the Stochastic Block Model used in this paper and discuss the community reconstruction problem. We discuss some of the relevant literature in Section \ref{related}. The system model describing the generation of correlated graphs and deanonymization attack is discussed in Section \ref{model}. Necessary conditions for the anonymity in the Stochastic Block Model are derived in Section \ref{necc}. In Section \ref{sec:recoveryregion}, we describe the existence of anonymized community recovery region and discuss approaches to boost the anonymity in networks. We consider the case of growing number of communities in Section \ref{sec:LargeComm}. We evaluate and relate the performance of a particular deanonymization algorithm to our results in Section \ref{sec:eval}. We conclude the paper with some remarks in Section \ref{sec:conclusion}.
\section{The Stochastic Block Model}\label{commreconstruct} 
Communities are an integral part of any social network. The community structure also plays an important role in many data analytics' applications. In Section \ref{related_comm}, we shed light on some applications of community reconstruction to emphasize its importance. In recent years, community detection/reconstruction problems have been extensively studied for the Stochastic Block Model (SBM) \cite{coja2010graph,decelle2011asymptotic,massoulie2014community, hajek2015achieving,abbe2014exact}.
SBM is a simple generalization of the Erd\"{o}s R\'enyi model that incorporates community structure.

The SBM is defined as follows. 
Suppose that $n$ vertices are partitioned into $C$ disjoint subsets, called communities.  
A symmetric $C \times C$ matrix, $P$, specifies edge probabilities: for two vertices $u$ and $v$ in communities $i = \mathcal{C}(i)$ and $ j =\mathcal{C}(j)$, $u$ and $v$ are adjacent with probability $P_{ij}$.
The presence of distinct edges is independent.
A special case of SBM is the {\em planted partition model} in which the entries of the probability matrix $P$ are a constant $p$ on the diagonal and a constant $q$ off the diagonal. Specifically, $P_{ij}=p$ if the nodes $i$ and $j$ are in the same community, else $P_{ij}=q$. It is assumed that $p>q$ as nodes in a community are relatively densely connected. Such a network is denoted by $SBM(n,p,q)$. 

There has been a series of results for the exact community recovery in the planted partition model. These studies assume a sparse regime where $p=a\frac{\log n}{n}$ and $q=b\frac{\log n}{n}$, where $a,b>0$ are some fixed constants.\footnote{Learning the community structure is harder for the sparse regime. Thus, it constitutes the more interesting case. } The case of two communities, $C=2$, was studied in \cite{abbe2014exact} in which Abbe et al. analyze the information-theoretic bounds for exact recovery and establish a phase transition phenomenon for the problem. Additionally, they propose a Semidefinite Programming (SDP) based algorithm for exact recovery of communities. Hajek et al. subsequently prove that the SDP algorithm is optimal, that is, it recovers the exact communities whenever it is theoretically possible to do so \cite{hajek2014achieving}. Hajek et al.  further extended Abbe et al.'s results to an arbitrary fixed $C$.  As this particular result is relevant to our derivations, we state it in Section \ref{sec:recoveryregion} (See Theorem  \ref{theorem:comm}).\\
\noindent{\bf Remarks.} The definition of a community varies with applications and algorithms \cite{girvan2002community,chen2006detecting}. Also the varying definitions result in the theoretical analysis becoming intractable. Although the SBM may not capture the community structure in the real world networks perfectly, it lands itself to tractable analysis. Apart from simplicity, it also captures one of the most important elements of communities, \emph{assortativity}. Hence, we focus on the SBM which has a clear definition of a community and the ground truth is available while evaluating an algorithm. Moreover, our results can be generalized to unequal sized communities which is more practical, but this scenario makes the analysis more involved without providing any new insights into the problem.

%% file: related_work.tex
\section{Related Work}
\label{related}
In this section, we discuss some of the important deanony-mization attacks followed by some applications of community detection in networks. Due to space limitation we only discuss the most relevant literature.
\subsection{Network De-anonymization} 
In \cite{pedarsani2011privacy}, Pedarsani and Grossglauser studied the deanony-mization problem for two correlated  Erd\"os R\'enyi random graphs. They assume that both the anonymized and auxiliary graphs are sampled from a common underlying Erd\"os R\'enyi random graph which results in structural correlation between the two graphs. They derive sufficient conditions on the model parameters under which the two graphs can be matched exactly. Specifically, they prove that the average degree only needs to grow slightly faster than the logarithm of order of the network to achieve perfect deanonymization. A similar problem was considered by Ji et al. in \cite{ji2014structural}. To generate the correlated graphs, the sampling process as in \cite{pedarsani2011privacy} was used, but the underlying graph is drawn from the configuration model \cite{newman2010networks}. They derive sufficient conditions on the model parameters for the perfect as well as partial deanonymization of networks. 

Ji et al. studied the role seed nodes play in assisting the exact and partial deanonymization process \cite{ji2015your}. They derived achievability thresholds for  for Erd\"os R\'enyi random graphs as well as graphs from arbitrary distribution models. They also evaluated their results on $24$ real world social networks and showed varying degree of vulnerability among the networks to the deanonymization attacks.

Yartseva et al. studied the performance of a specific algorithm, the so called percolation graph matching algorithm, for deanonymizing Erd\"os R\'enyi random graphs \cite{yartseva2013performance}. Starting with some seed nodes, the algorithm incrementally maps remaining pair of nodes using a thresholding criterion. They prove sufficient conditions on model parameters which enable this algorithm to match the networks almost perfectly. A phase transition in the initial seed set size is established.

Narayanan and Shmatikov proposed a two-stage algorithm to deanonymize a network again when the adversary has  access to  an auxiliary network whose user base overlaps partially with that of the anonymized network \cite{narayanan2009anonymizing}. After the seed identification in the first phase, the algorithm propagates this information and identifies further nodes in the second phase. They show that the users who have accounts on both Twitter (anonymized) and Flickr (auxiliary) can be deanonymized with only a $12\%$ error rate.  Nilizadeh et al. enhance the performance of the algorithm in \cite{narayanan2009anonymizing} using the community structure of the network \cite{nilizadeh2014community}. Their attack proceeds as follows. First, communities are detected and mapped in both the anonymized and auxiliary graphs. Subsequently, more seeds are identified within the communities and deanonymization is performed for each pair of communities using already existing algorithm of \cite{narayanan2009anonymizing}. This algorithm is again run on the whole graph in case some nodes are not mapped in the previous steps. The authors show empirically that this algorithm helps in boosting the deanonymization process on a specific graph derived from Twitter. The results in  \cite{narayanan2009anonymizing,nilizadeh2014community}  are great examples of why deanonymization poses a real threat to users, but they do not provide fundamental limits or even insights into when the deanonymization problem is hard or easy.

\subsection{Community Detection Applications}
\label{related_comm}
\subsubsection{Privacy Control \cite{jones2010feasibility}}
Information sharing  in an OSN, like photos, statuses, emotions and location, is a common practice for individuals using the network. An individual's social contacts may fall into various categories like family, friends, colleagues and even finer subgroups. However, users may want to share their information among a particular group of contacts only. Hence, it becomes important that users are able to selectively share their information over such networks. 

Recently, community reconstruction has been suggested as a tool to solve this problem. Community is a group of tightly connected users where the tightness may be measured by various metrics for different types of groups. For example, a close friendship community is characterized by frequent message sharing. These metrics may be exploited to automatically detect and label communities which enable users to selectively share their information.

\subsubsection{Sampling in OSNs \cite{yoon2015community}}
The popularity of OSNs has grown beyond imagination in the last few years and the size of these networks has grown into millions of users. The data generated by these networks is tremendous and very useful for research and other purposes. As the networks become large, it becomes difficult to analyze the properties of an entire network. 

The sampled graphs should be representative of the original graph is terms of both  the local and global properties like degree distributions, node-edge ratio etc. 
Community reconstruction plays a crucial role to create such representative samples of the original graph. 
In a typical application, the hierarchical community structure of anetwork is reconstructed.
Sampling is done based on the observed communities, ensuring that local properties are preserved in the sampled version. 
Then, in a bottom-up fashion, these sampled subgraphs are linked together to form a bigger graph. 
\subsubsection{Viral Meme Prediction \cite{weng2014predicting}} 
A meme is ``an idea, or style that spreads from person to person within a culture ". They are similar to infectious diseases within a network. Out of many memes generated each day only a few of them go viral within a network. This viral behavior is of value to advertising and marketing businesses. 

There are many factors which contribute to the popularity of memes such as timing, point of beginning and others. Recently, the underlying network structure, specifically community structure, has also been identified as an important feature. Community structure exhibits two important phenomena: social reinforcement and homophily. These features expose community members to a meme more often. This results in higher rates of adoption for a meme. Hence a meme may become more popular within a community with strong social reinforcement and homophily. The features like number of initially infected communities, distribution of infected users across communities, and intra community interaction strength are used to predict the viral memes. 



%% file: model.tex
\section{System Model}\label{model}
In this section, we discuss the system model, followed by the description of the deanonymization attack and derivation of the Maximum a Posteriori estimator. First, we discuss a few preliminaries.
\subsection{Preliminaries}
For a graph $G$, let $V(G)$ and $E(G)$ be the node and edge sets respectively. 
Let $[n]$ denote the set $\{1,\cdots,n\}$.
All of the $n$-vertex graphs that we consider will have vertex set $[n]$.
These numerical vertex labels should not be confused with the vertex labels (or alternatively tags) coming from the problem domain.
Examples of these tags include names and private information.
To anonymize the graph $G_2$, we need to remove the relationship between the numerical vertex labels and the user identities.
To do this, we will apply a uniform random permutation to the numerical vertex labels.
We will always think of a permutation as a function defined from $[n] \to [n]$.
We denote the collection of all two element subsets of $[n]$ by $\binom{[n]}{2}$.
The edge set of a graph $G$ is $E(G) \subseteq \binom{[n]}{2}$.
The community label of a node $i \in V(G)$ is denoted by $\mathcal{C}(i)\in [C]$, where $C$ is the number of communities.
\subsection{Generative Model for Correlated Graphs}
Recall the description of the problem in Section~\ref{intro}. 
There are two $n$-vertex graphs, $G_1$ and $G_2$.
An attacker has access to a pair of correlated graphs $G_1$ and $\pi(G_2)$, both  graphs defined on the same vertex set, denoting the identities of the users. 
Here $\pi$ is a uniformly random permuation of $[n]$.
The vertices of the {\em auxiliary} graph $G_1$, available to the attacker, are tagged with user identities but not any sensitive information.
The vertices of the {\em sensitive} graph $G_2$, not available to the attacker, are tagged with sensitive information.
The {\em anonymized} graph $\pi(G_2)$ is available to the attacker, but the numerical vertex labels contain no information about user identities because $\pi$ is a uniformly random permutation.
Vertex $i$ in $G_1$ and vertex $i$ in $G_2$ correspond to the same user, so given $G_1$, the numerical vertex labels in $G_2$ reveal the user identities.


To generate two correlated graphs, $G_1$ and $G_2$, the following mechanism is used. 
This is essentially the same model which was previously used in \cite{pedarsani2011privacy,ji2014structural,ji2015your}.
The two correlated graphs are assumed to be sampled from a random underlying graph $G$ on the same set of vertices. 
Specifically, $G$ is distributed as $SBM(n,p,q)$ with $C$ equally sized communities. 
Each edge $\{i,j\}\in E(G)$ is included in $G_1$ independently with probability $s_1$. 
The graph $G_2$ is created similarly using sampling probability $s_2$ and these choices are independent of all choices made to create $G_1$. 
As a result, $G_1$ is $SBM(n,ps_1,qs_1)$, $G_2$ is $SBM(n,ps_2,qs_2)$, and $E(G_1)$ and $E(G_2)$ are correlated but in general not equal. 

\subsection{Attack Model}\label{sec:attackmodel}
Recall that, $G_1$ is the \textit{auxiliary} graph and $G_2$ the \textit{sensitive} graph. 
An adversary aims to deanonymize $\pi(G_2)$ using $G_1$. 
A deanonymization attack can be described as a mapping from the nodes of $G_1$ to the nodes of $\pi(G_2)$, i.e. a map $\hat{\pi}: [n] \to [n]$. 
A successful deanonymization attack is the true mapping $\hat{\pi} = \pi$.
In that case, we say that the network $G_2$ is deanonymized exactly. 

When true community labels exist in the graphs, we assume that the adversary knows the true labels of all vertices in both graphs.
\footnote{As we are interested in converses, considering a stronger adversary does not pose a problem. 
In fact, given the anonymized graph $\pi(G_2)$, the adversary must be able to perform community detection with high probability (perform the intended data analytics)
.
}
We say that a permutation preserves the community structure if it maps vertices only to other vertices with the same community label.
That is, a permutation $\pi$ is community preserving if $\forall$ $i \in [n]$, $\mathcal{C}(\pi(i)) = \mathcal{C}(i)$.
Because the adversary can recover the community labels in both graphs, or equivalently can compute both $\mathcal{C}(i)$ and $\mathcal{C}(\pi(i))$, they can learn some information about the permutation $\pi$.
The adversary can group the vertices of $\pi(G_2)$ by community, producing another graph $\pi'(G_2)$ such that $\pi'$ preserves communities.
In other words, anonymizing $G_2$ using a uniformly random $\pi$ does not create additional uncertainty for the adversary beyond what would be created by a permutation that preserves the community structure.
So our analysis considers only the latter type of permutation.

\label{sec:MAP}
An adversary is presented with a statistical estimation problem.
By definition, the Maximum a Posteriori (MAP) estimator minimizes the adversary's probability of error. 
So, if the MAP estimator does not recover the true permutation with high probability, then no other estimator can succeed.
We also assume that all the permutations used to anonymize $G_2$ are equiprobable. 
Hence, the MAP estimator is same as the Maximum Likelihood estimator.

If we fix any randomized estimation procedure, then the adversaries estimate $\hat{\pi}$ become a random variable.
It will be more convenient to let $\Phi = \hat{\pi} \circ \pi^{-1}$ and work with the random permutation $\Phi$ rather than $\hat{\pi}$ directly.
In a successful attack, $\Phi = I$, the identity permutation.
The reason that $\Phi$ is more convenient is that $\Phi$ is independent of $\pi$.
For fixed $G_1$ and $G_2$, any change in $\pi$ results in a corresponding change in $\hat{\pi}$ and this does not change $\Phi$.

The MAP estimator for this problem can be derived as follows. 
We need to compute the likelihood of the posterior probability of a mapping $\Phi$, given the observed graphs $G_1$ and $ \pi(G_2)$, that is, $P[\Phi = \phi | G_1, \pi(G_2)]$.
Note that a particular mapping $\phi : [n] \to [n]$ induces a mapping $\sigma_{\phi}: \binom{[n]}{2} \to \binom{[n]}{2}$ on the node pairs.
Define 
\begin{equation}\label{eq:symmedge
diff}\mathcal{S}_{\phi}=(E(G_1) \cup \sigma_{\phi}(E(G_2))) \setminus (E(G_1) \cap \sigma_{\phi}(E(G_2))),
\end{equation}
 which is the symmetric edge difference of the two edge sets. 
Also for any graph with  community labels, define the following two sets,
\begin{align}
\mathcal{E}^{in}_{\phi}=&\{(i,j)\in\mathcal{S}_{\phi}: \mathcal{C}(i)=\mathcal{C}(j)\}\label{intra}\\
\mathcal{E}^{out}_{\phi}=&\{(i,j)\in\mathcal{S}_{\phi}: \mathcal{C}(i)\neq\mathcal{C}(j)\},\label{inter}
\end{align}
the symmetric edge difference sets corresponding to the intra and inter community edges respectively. 
 
For SBM graphs defined in Section \ref{commreconstruct}, an easy computation shows that the posterior probability $P[\Phi = \phi | G_1, \pi(G_2)] $ is proportional to $c_1^{|\mathcal{E}_{\phi}^{in}|}c_2^{|\mathcal{E}_{\phi}^{out}|}$, where
\begin{align*}
c_1 =& \frac{p(1-s_1)(1-s_2)}{1-p+p(1-s_1)(1-s_2)}\\ 
c_2 =& \frac{q(1-s_1)(1-s_2)}{1-q+q(1-s_1)(1-s_2)}
\end{align*}
Note that $c_1,c_2\leq 1$. Then the MAP estimator is given by
\begin{equation*}
\arg\min_\phi \log\left(\frac{1}{c_1}\right)|\mathcal{E}_{\phi}^{in}|+\log\left(\frac{1}{c_2}\right)|\mathcal{E}_{\phi}^{out}|,
\end{equation*}
the mapping which minimizes a linear combination of $|\mathcal{E}_{in}|$ and $|\mathcal{E}_{out}|$ weighted by fixed positive coefficients.

Results in the next section find conditions under which the MAP estimator fails with high probability.

%% file: theorems.tex
\section{Conditions for Anonymity}\label{necc}
We analyze the anonymity of the graph $\pi(G_2)$ with the SBM community structure for the attack model described in Section \ref{sec:attackmodel}. Recall that the attacker has access to a correlated graph $G_1$ with known vertex labels. In this section we consider the problem for arbitrary fixed number of communities $C$. We generalize the result in Section \ref{sec:LargeComm} to study the impact of growing number of communities.
The following two lemmas are useful in proving the main result.
To avoid making the proof too technical and for ease of presentation, we present this result for the case in which the community sizes are equal. 
\begin{lemma}
\label{lemma:expected}
Let $p = \frac{a \log n}{n}$ and let $q = \frac{b \log n}{n}$.
Let $G\sim SBM(n,p,q)$ with $C$ equally sized communities. 
Let $X_k$ be the number of isolated vertices in community $k$ of $G$.
If $\frac{a+(C-1)b}{C} < \alpha$, then $E[X_k] \geq \frac{n^{1 - \alpha}}{C}(1 - o(1))$.
Additionally, $\mathbb{P}\left[X_k \leq \frac{n^{1-\alpha}}{2C}\right] \to 0$.
\end{lemma}
\begin{proof}
Let $n' = n/C$.
Define a random variable $l_i$ as an indicator of the event that node $i$ is isolated. Then 
\begin{equation*}
E[l_i]=(1-p)^{n'-1}(1-q)^{(C-1)n'}
\end{equation*}
Let $S_k = \{i : \mathcal{C}(i) = k\}$, the vertices of community $k$.
Then $X_k = \sum_{i \in S_k} l_i$ denotes the total number of isolated nodes in a particular community. 
The expected value of $X_k$ goes to infinity:
\begin{IEEEeqnarray*}{rCl}
E[X_k]
&\geq& n'(1-p)^{n'}(1-q)^{(C-1)n'}\\
&=& n'\left(1 + \frac{p}{1-p}\right)^{-n'}\left(1 + \frac{q}{1-q}\right)^{-(C-1)n'}\\
&\geq& n' \left(\exp\left(\frac{p}{1-p}\right)\right)^{-n'} \left(\exp\left(\frac{q}{1-q}\right)\right)^{-(C-1)n'}\\
&=& \frac{1}{C} \exp\left(\log n - \frac{n'p}{1-p} - \frac{(C-1)n'q}{1-q}\right)\\
&=& \frac{1}{C} \exp\left(\log n \left(1 - \frac{a}{C(1-p)} - \frac{(C-1)b}{C(1-q)}\right)\right)\\
&\geq& \frac{1}{C} \exp\left(\log n \left(1 - \frac{\alpha}{1-p}\right)\right)\\
&=& \frac{n^{1-\alpha}}{C}(1 - o(1))
\end{IEEEeqnarray*}
Note that $\mathbb{E}[X]\to \infty$ if $1 > \frac{a+(C-1)b}{C}$. 
We want to show that $Var(X)$ is of the same order of $\mathbb{E}[X]$. 
Then we can use
\begin{equation}
\mathbb{P}[X\leq \frac{1}{2}\mathbb{E}[X]]\leq 4\frac{Var[X]}{\mathbb{E}[X]^2} = \frac{4 \mathbb{E}[X](1+o(1))}{\mathbb{E}[X]^2} \to 0
\end{equation}
Now we need to show that the variance is indeed of the order of expectation.
\begin{align}
Var[X]&= \sum_{i\in S_k}\sum_{j \in S_k} \mathbb{E}[l_il_j]-\mathbb{E}[X_k]^2\nonumber
\end{align}
For $i,j \in S_k$, $i \neq j$, $\mathbb{E}[l_il_j]$ is equal to
\begin{equation*}
(1-p)^{2n'-3}(1-q)^{2(C-1)n'} = \frac{E[l_i]^2}{1-p} = \frac{E[X_k]^2}{(n')^2(1-p)}
\end{equation*}
Hence, using $E[X_k] \to \infty$ we have
\begin{align*}
Var[X_k] &= E[X_k]  + \frac{n'(n'-1)E[X_k]^2}{(n')^2(1-p)} + E[X_k]^2\\
&= E[X_k] (1+o(1))
\end{align*}
Hence we have $\mathbb{P}[X\leq\frac{1}{2}\mathbb{E}[X]]\to 0$. 
This means that with probability going to 1, the number of isolated vertices in a community goes to infinity, growing as $n^{1-\alpha}$. 
This completes the proof.
\end{proof}

Recall the MAP decision rule of Section \ref{sec:MAP} which selects the permutation $\Phi $ which maximized the posterior probability $P[\Phi  | G_1, \pi(G_2)]$. Recall that with our choice of notation, if the true permutation is identified then $\Phi =I$, the identity permutation.
Next lemma shows that any permutation in the automorphism group of the intersection graph $G_1~\cap~G_2$ achieves at least as large of a posterior probability as the true permutation $I$.

\begin{lemma}
\label{permute}
Let $G_1$ and $G_2$ be the correlated SBM graphs. 
Let $\text{Aut}(G_1\cap G_2)$ denote the automorphism group of $G_1~\cap~G_2$.
If $\phi \in \text{Aut}(G_1\cap G_2)$ preserves the community structure, then 
\begin{equation*}
P[\Phi = \phi | G_1, \pi(G_2)] \geq P[\Phi = I | G_1, \pi(G_2)].
\end{equation*}
\end{lemma}
\begin{proof}
Consider a vertex pair $\{i,j\}\in {[n]\choose 2}$. 
Suppose $\mathcal{C}(i)=\mathcal{C}(j)$. 
Note that $\{i,j\}$ can only affect the intra community edge set symmetric difference defined in \eqref{intra}. If $\{i,j\}\in E(G_1 \cap G_2)$ then its contribution to $|\mathcal{E}_{\phi}^{in}|$ and $|\mathcal{E}_{I}^{in}|$ is equal. This is because both $\phi$ and $I$ are in $\text{Aut}(G_1\cap G_2)$, so by definition the edges in $ G_1\cap G_2$ remain intact.
If $\{i,j\} \notin E(G_1 \cap G_2)$, then there are two possibilities. 
If $\{i,j\} \in E(G_1 \cup G_2)$ then its contribution to $|\mathcal{E}_{I}^{in}|$ is $1$ and to $|\mathcal{E}_{\phi}^{in}|$ is either $0$ or $1$. 
If $\{i,j\} \notin E(G_1 \cup G_2)$ then its contributions to both is $0$. 
Hence, $|\mathcal{E}_{I}^{in}|\geq |\mathcal{E}_{\phi}^{in}|$.
  
Alternatively, suppose $\mathcal{C}(i)\neq\mathcal{C}(j)$. Note that $\{i,j\}$ can only affect the inter community edge set symmetric difference defined in \eqref{inter}. The rest of the arguments are similar to the previous case. Hence, $|\mathcal{E}_{I}^{out}|\geq |\mathcal{E}_{\phi}^{out}|$.
  
Thus $P[\Phi = \phi | G_1, \pi(G_2)] \geq P[\Phi = I | G_1, \pi(G_2)]$.
\end{proof}
\begin{theorem}[SBM Converse]
\label{thm:converse}
If $\frac{(a+(C-1)b)s_1s_2}{C} < 1-\alpha$ then with probability $1-o(1)$ at least $n^{\alpha}/2$ vertices of the graph $\pi(G_2)$ cannot be deanonymized. Furthermore, these vertices are all mutually confusable, so there are at least $(1-o(1)) \alpha \log_2 n$ bits of uncertainty about the identity of these vertices.

In particular, if $\frac{(a+(C-1)b)s_1s_2}{C} < 1$ then with probability $1-o(1)$, $\pi(G_2)$ cannot be deanonymized exactly using $G_1$.

\end{theorem}
\begin{proof}
Note that $G_1\cap G_2$ $\sim$ $SBM(n,ps_1s_2,qs_2s_1)$ with the community labels known. 
Let $X_k$ be the number of isolated vertices in community $k$ of $G_1\cap G_2$.
By Lemma \ref{lemma:expected}, with probability $1-o(1)$, $X_k = \Omega(n^{1-\alpha})$.
Any permutation that moves only these isolated vertices, preserving community structure, is an automorphism of $G_1\cap G_2$.
By Lemma \ref{permute}, the adversary's posterior probability of such a permutation is at least as large as the posterior probability of the identity.
Thus the MAP estimator for the whole permutation $\phi$ succeeds with probability at most $\frac{1}{|Aut(G_1\cap G_2)|}$.
As long as $\alpha > 0$, $X_k \to \infty$ for all $k$ and $|Aut(G_1\cap G_2)| \to \infty$.
For some isolated vertex $i$, the MAP estimator for $\phi(i)$ succeeds with probability at most $\frac{1}{|X_{\mathcal{C}(i)}|} = n^{-\alpha}/2$.
With probability $1-o(1)$, there are at least $(1-o(1)) \alpha \log_2 n$ bits of uncertainty about the identity of a particular isolated vertex.
\end{proof}
The converse implies that sufficiently sparse pairs of SBM graphs cannot be exactly deanonymized.
Next we provide a nearly matching achievability region, i.e., a sufficient condition for deanonymizing graph  $\pi(G_2)$ and $G_1$. The importance of this result is that it illustrates the strength of our converse in Theorem~\ref{thm:converse}.
\begin{theorem}
\label{thm:ach}
Let $p = \frac{a \log n}{n}$ and let $q = \frac{b \log n}{n}$.
Let $G\sim SBM(n,p,q)$ and let $G_1$ and $G_2$ be subsampled from $G$ with probabilities $s_1$ and $s_2$.
If $\frac{(a+(C-1)b)s_1s_2}{C} > 2$, then there is an algorithm which exactly recovers $\pi$ with probability $1 - o(1)$ given $\pi(G_2)$, $G_1$, and the true community labels for each of these graphs.
\end{theorem}
This proof is omitted due to space constraints.
Recent work investigates the analogous problem for Erd\H{o}s R\'{e}nyi graphs \cite{cullina2016improved}.
Theorem~\ref{thm:ach} follows from fairly straightforward adaptation of the argument used there.
The bound in Theorem~\ref{thm:ach} has the same dependence on $a$, $b$, $s_1$, and $s_2$ as Theorem~\ref{thm:converse}.
In the case of exact deanonimization, the threholds differ only by a constant factor of 2.
Consequently, the conditions that we require our anonymized graph to satisfy are not excessively conservative.

\section{Community Recovery}\label{sec:recoveryregion}
Our converse identifies a region on parameters of the model that guarantees no adversary can deanonymize $\pi(G_2)$, the anonymized graph, given access to the auxiliary graph $G_1$. The anonymized graph $\pi(G_2)$ is useful to the third parties only if they are still able to perform community detection in some portion of the identified region. 

In this section, we show that there indeed exists a region in which community detection succeeds but deanonymization fails. To do so, we will combine Theorem~\ref{thm:converse} with a recent result regarding the feasibility of exact recovery of community labels in an SBM graph. This result is tight, but we only need the achievability part.
The following theorem was proved for the two-community case by Abbe et al. \cite{abbe2014exact} and independently by Mossel et al. \cite{mossel2014consistency}, both in 2014.
Hajek et al. generalized the result to arbitrary fixed $C$ \cite{hajek2015achieving}.
\TODO{cites}
\begin{theorem}\cite{hajek2015achieving}
\label{theorem:comm}
Let $G \sim SBM(n,p,q)$ with $C$ communities, where $p = \frac{a \log n}{n}$ and $q = \frac{b \log n}{n}$.
If $\sqrt{a} - \sqrt{b} > \sqrt{C}$, then there is an algorithm that exactly recovers the community labels of $G$ with probability $1 - o(1)$.
\end{theorem}

\begin{corollary}
	\label{corollary1}
As long as $s_1 < 1$, there are parameters $s_2$, $a$ and $b$ such that $\pi(G_2)$ cannot be deanonymized exactly using $G_1$ but exact community recovery is possible in $G_2$.
\end{corollary}
\begin{proof}
From Theorem~\ref{theorem:comm} we have the inequality $\sqrt{as_2} - \sqrt{bs_2} > \sqrt{C}$. 
From Theorem~\ref{thm:converse} we have $(a+(C-1)b)s_1s_2 < C$. This region is not empty. For instance, for $b \rightarrow 0$, $a$  must lie in the range $\frac{C}{s_2} < a < \frac{C}{s_1s_2}$.
\end{proof}

\begin{figure}[t]
	\includegraphics[width=0.5\textwidth]{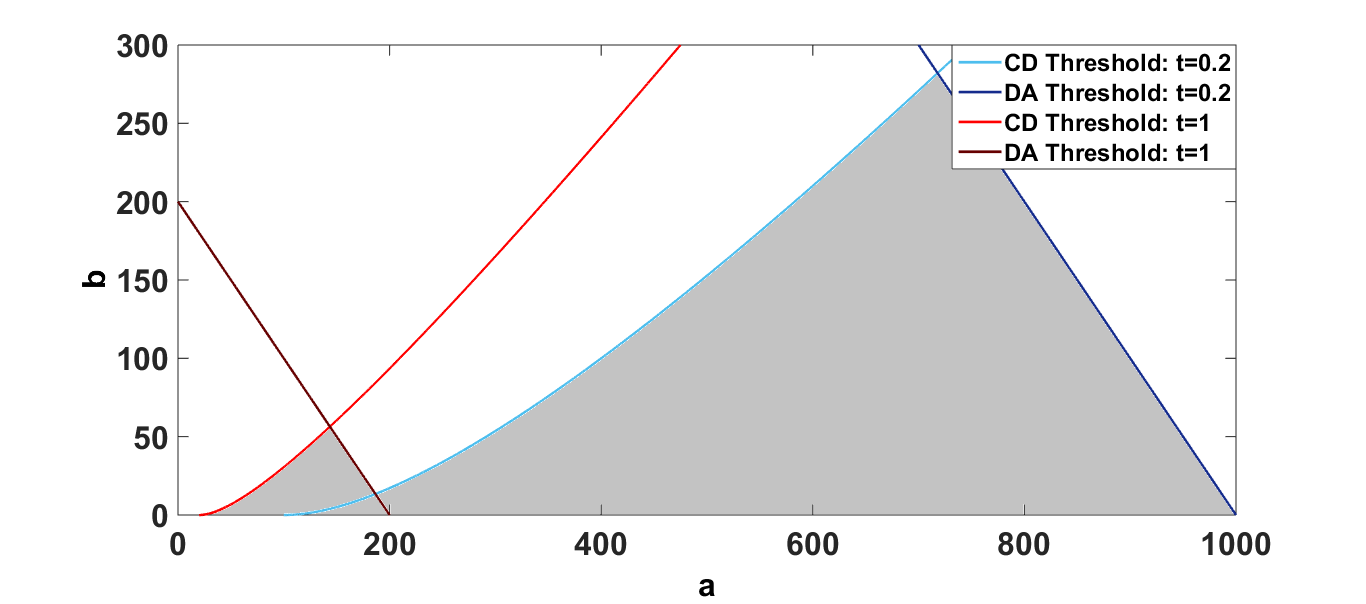}
	\vspace{-20pt}\caption{Plot showing the safe region (shaded) for $s_1=0.1$, $s_2=0.5$ and $C=2$. Here CD denotes Community Detection and DA denotes Deanonymization.}\label{intersect}
\end{figure}

Instead of releasing $\pi(G_2)$ in its original form, we could release a edge-subsampled version.
By subsampling a graph, we mean randomly including each edge of the graph independently with some probability $t$.
In many cases, some choice of $t$ results in a graph that falls into the safe region. 

Then the necessary condition for community detection becomes
$\sqrt{as_2t} - \sqrt{bs_2t} > \sqrt{C}$
and the condition preventing exact deanonymization becomes
$(a+(C-1)b)s_1s_2t < C.$
This region is depicted in Figure \ref{intersect} for two values of $t$. For $t=1$, we recover the region corresponding to Corollary \ref{corollary1}. It can be seen that subsampling with $t=0.2$ results in a substantial increase in the parameter space of interest.
The subsampling idea works for any number of communities, but in the two-community case, we have a very simple condition. 
\begin{corollary}
For $C=2$, if exact recovery of communities is possible in $G_2$ and $\left(\frac{a-b}{a+b}\right)^2 + (1-s_1)^2 > 1$, then there is some subsampling probability $t$ such that the $t$-subsampled version of $\pi(G_2)$ still allows community detection but cannot be deanonymized given $G_1$.
\end{corollary}

The simple structure of this region is depicted in Figure~\ref{fig:subsample}.

\begin{figure}[t]
	\includegraphics[width=0.5\textwidth]{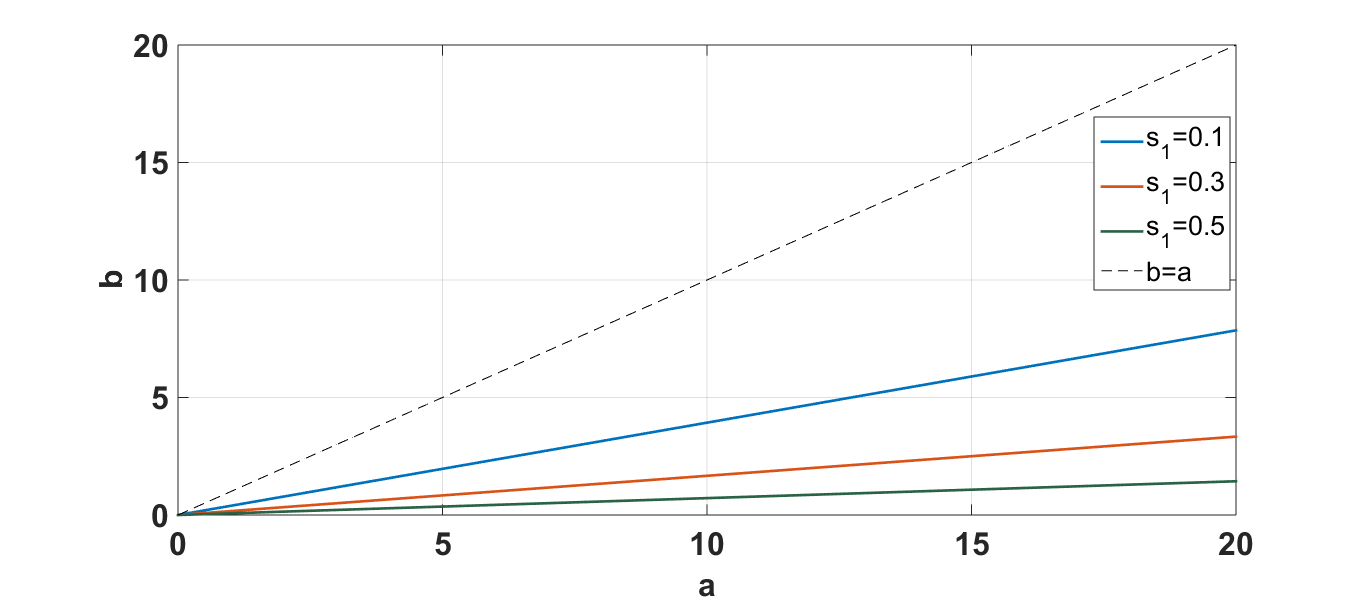}\vspace{-11pt}
	\caption{Plot depicting the parameter space(below the curves), for which exact community recovery is possible but exact deanonymization is impossible for some under sampling probability $t$, for various values of $s_1$ and $C=2$.}\label{fig:subsample}
\end{figure}

The ratio $\frac{a-b}{a+b}=\frac{p-q}{p+q}$ measures the strength of the community structure in the SBM graph.
Note that because $0 \leq b \leq a$, we have $0 \leq \frac{a-b}{a+b} \leq 1$.
Unsurprisingly, when the community structure is stronger, fewer edges of $G_2$ must be preserved to allow community recovery.
This allows us to create a greater degree of anonymity.
The other factor $(1-s_1)$, measures the amount of ground truth information included in the auxiliary graph.
As more of this information is publicly available, it becomes harder to produce an anonymized version of $G_2$.

To produce a graph that can be published, we need to find a parameter range where there are many isolated vertices in $G_1 \cap G_2$ and none in $G_2$ (because isolated vertices in $G_2$ prevent exact community recovery).
When $s_1 = 1$, we have $E(G_2) \subseteq E(G_1)$, so $G_2 = G_1 \cap G_2$ and and it becomes impossible to produce a safe graph, regardless of the strength of the community structure.

%% file: communities.tex
\section{Sublinear communities}
\label{sec:LargeComm}
So far, we have only considered graphs with a constant number of communities, or equivalently communities with a number of vertices linear in $n$.
In real world graphs, community structure arises for a variety of reasons.
For example, a community derived from common interest in some popular media franchise could easily have linear size.
As the overall network grows, the probability of a new user being a member of this community would be close to constant.
In contrast, communities that arise from local real-world interactions will generally be sub-linear in size.

We have assumed that the adversary is capable of detecting the community structure in both the public and anonymized graphs and correctly matching a community in one graph to a community in the other.
If $C$ is constant, the community level matching reduces the anonymity of a single vertex by an asymptotically negligible amount.
Without the community level matching, $\log_2 n$ bits are required to describe the corresponding vertex in other graph.
With it, $\log_2 (n/C) = (1-o(1)) \log_2 n$ bits are required.
Because of this, the asymptotic threshold in Theorem~\ref{thm:converse} does not depend on $C$, 
When the number of communities is growing and the size of a typical community is sub-linear, the community level matching contains a non-negligible amount of information about each vertex identity.
Consequently, when $C \to \infty$, the threshold for anonymity does depend on the growth rate of $C$.
Our converse argument depends on the existence of community preserving automorphisms of $G_1 \cap G_2$.
If the number of communities is growing with $n$, it is possible to have a large number of total isolated vertices in the graph, but still no communities with multiple isolated vertices.
For this regime, we are not aware of results giving the conditions under which  community recovery is possible, but we derive the following converse for the deanonymization problem.
\begin{theorem}
\label{thm:converse-comm}
Let the number of communities be $C = n^{\beta}$ for some $0 < \beta < 1$.
If $\frac{(a+(C-1)b)s_1s_2}{C} < 1-\alpha-\beta$ then with probability $1-o(1)$ at least $n^{\alpha}/2$ vertices of the graph $\pi(G_2)$ cannot be deanonymized. Furthermore, these vertices are all mutually confusable, so there are at least $(1-o(1)) \alpha \log_2 n$ bits of uncertainty about the identity of these vertices.
\end{theorem}
\begin{proof}
Let $X_k$ be the number of isolated vertices in community $k$ of $G_1\cap G_2$.
By Lemma \ref{lemma:expected}, with probability $1-o(1)$, $X_k = \frac{n^{1-\alpha}}{2n^{\beta}} = \Omega(n^{1-\alpha-\beta})$.
The remainder of the proof is parallel to that of Theorem~\ref{thm:converse}.
\end{proof}

This theorem implies that to achieve the same level of uncertainty about identities of vertices as in the constant community case (i. e., $(1-o(1)) \alpha \log_2 n$ bits), a more conservative threshold is needed (Note the shift by $\beta$).

%% file: experiments.tex
\begin{figure*}[th]
	\center{
		\begin{tabular}{@{}c@{}c@{}}
			\includegraphics[scale=0.25]{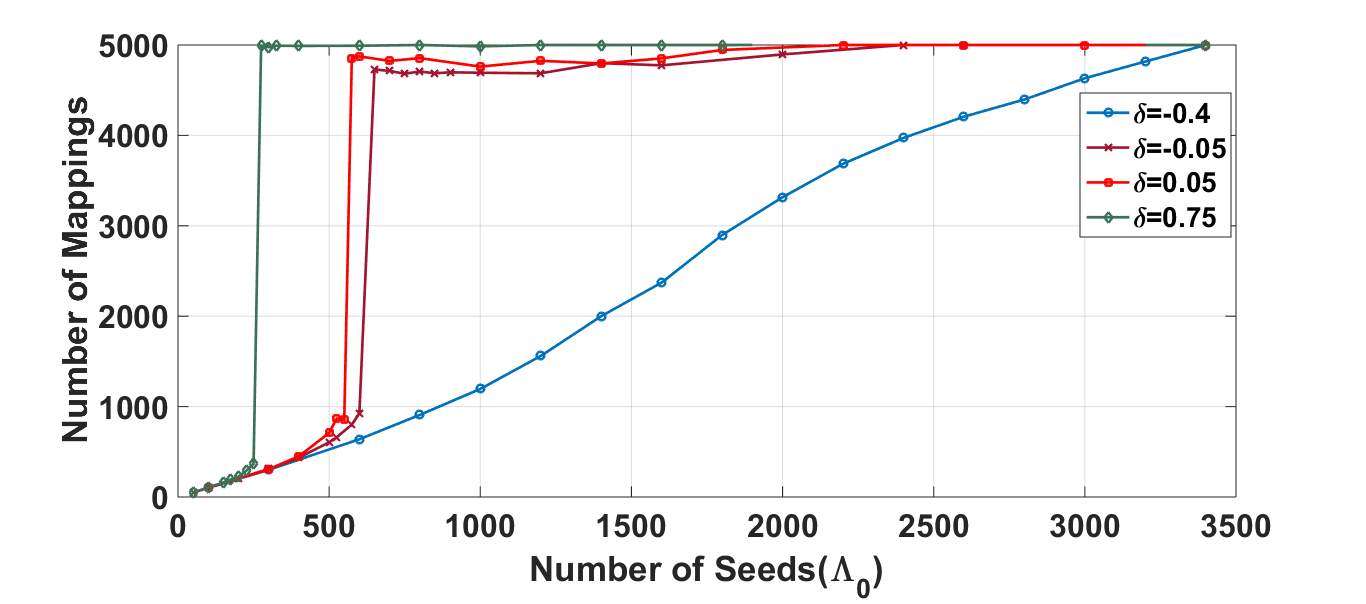}&\includegraphics[scale=0.25]{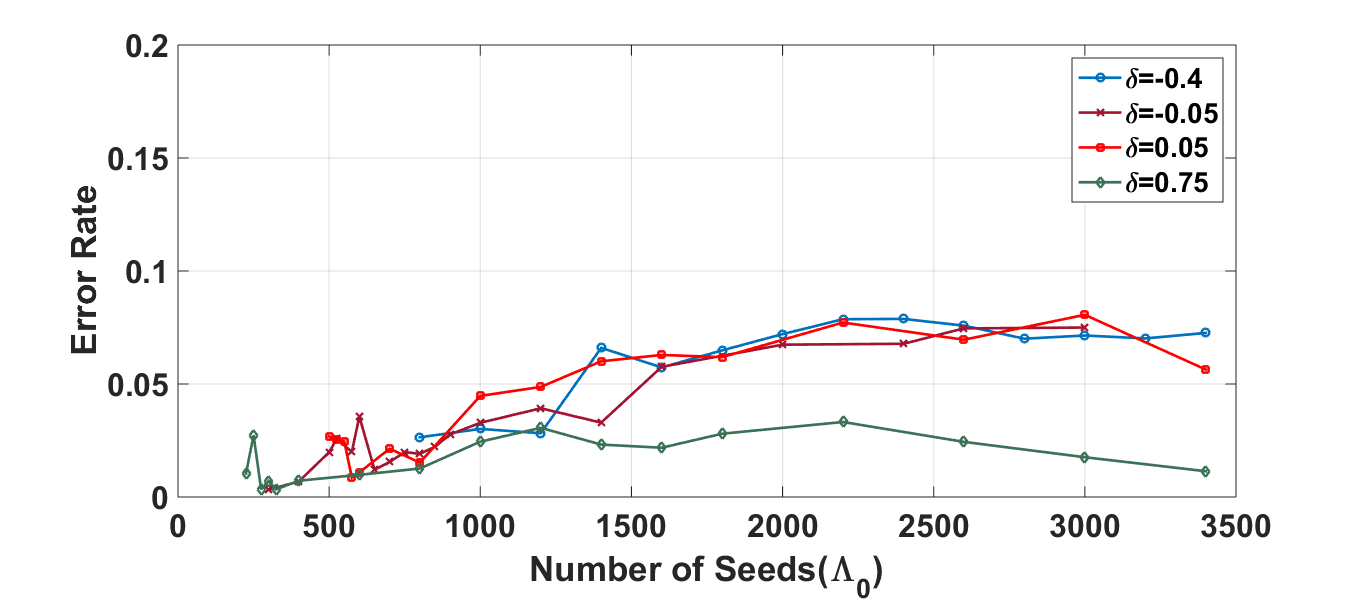}\\
			(a) & (b) \end{tabular}} \vspace{-10pt}\caption{{Plots for the Percolation method based Deanonymization Algorithm with $n=5000$ and $r=4$. (a) Total number of mapped nodes vs Number of seeds, and (b) Error Rate vs Number of Seeds
		}}\label{r4}
	\end{figure*}
		\begin{figure*}[th]
			\center{
				\begin{tabular}{@{}c@{}c@{}}
					\includegraphics[scale=0.25]{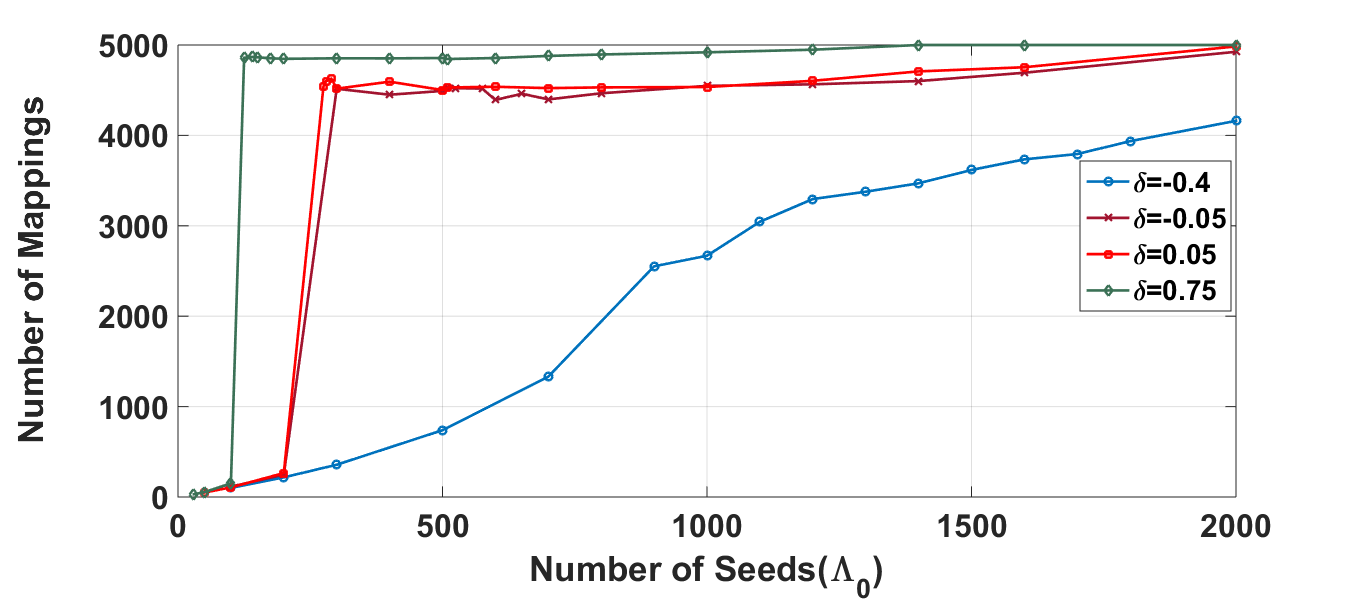}&\includegraphics[scale=0.25]{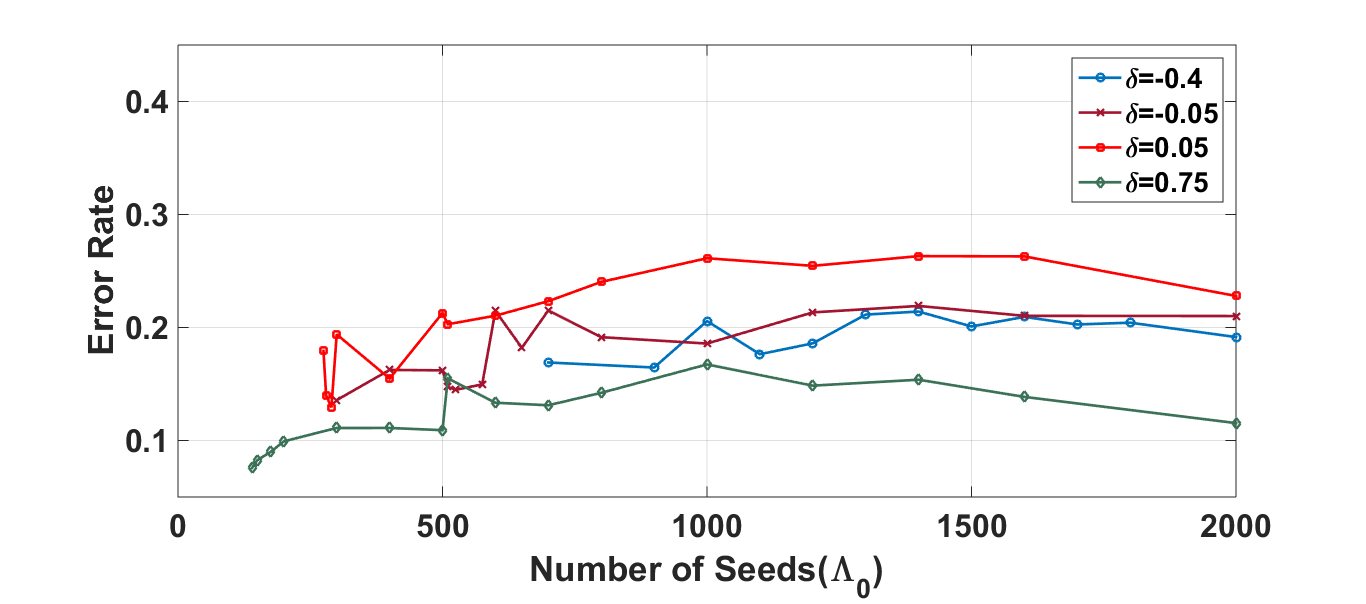}\\
					(a) & (b) \end{tabular}} \vspace{-10pt}\caption{{Plots for the Percolation method based Deanonymization Algorithm with $n=5000$ and $r=3$. (a) Total number of mapped nodes vs Number of seeds, and (b) Error Rate vs Number of Seeds
				}}\label{r3}
			\end{figure*}
\section{Experimental Results}\label{sec:eval}
In this section, we study the utility-privacy trade-off for both synthetic and real networks. The aim of this section is two-fold. First, we want to show that when the conditions of our converse are satisfied, most of the vertices in the network still remain anonymized. Second, we aim to demonstrate the existence of real networks which support community reconstruction without leaking the privacy of most of the users. 
\subsection{Results for SBM}\label{expsbm}
We consider the percolation based deanonymization algorithm proposed in \cite{yartseva2013performance}, when an adversary knows the community partition in the networks. The choice of this algorithm is motivated by the fact that its performance is guaranteed for random graphs. Other algorithms in the literature are heuristics based and their performance is highly dataset dependent. Because the structure of an SBM network is quite uniform, the structural properties used by the heuristic algorithms are present only in a very few locations.

The percolation algorithm starts with $\Lambda_0$ number of seed nodes, and incrementally maps the remaining pair of nodes, using a thresholding criteria controlled by parameter $r\geq2$. A large value of $r$ ensures a smaller deanonymization error but requires large number of seeds to percolate. Conversely, small values of $r$ make the percolation easier but increase the error rates. We analyze the performance of this algorithm on the networks drawn from the SBM family with two communities. We provide the algorithm with $\Lambda_0$ number of randomly selected seed nodes. In practice, the algorithm has to identify the seeds correctly, so our setting is helping the performance of the algorithms. To make use of the community structure, we only allow those mappings which match nodes belonging to the same community.

The underlying graph is drawn from the SBM distribution with two communities, that is, $G\sim SBM(n,a\frac{\log n}{n},b\frac{\log n}{n})$. Graphs $G_1$ and $G_2$ are generated using sampling probabilities $s_1$ and $s_2$ respectively. We also sub-sample the private graph $G_2$ with probability $t$. Define the \textit{offset} for the parameter space, measuring the distance of parameters from the threshold for deanonymization, by, $\delta=\frac{(a+b)s_1s_2t}{2}-1$. Note that, $\delta<0$ corresponds to the case when exact de-anonymization is impossible. We are interested in the performance of the algorithm with varying values of the $\delta$.

To generate the datasets, the parameters are fixed as $n=5000$, $a=20$ and $b=5$. We vary the values of the sampling probabilities $s_1$,$s_2$, and $t$ to tune the parameter $\delta$. The parameters are tuned such that the community structure is preserved perfectly, that is, parameters are in the regime where exact community recovery is possible using SDP. We first evaluate the algorithm for the thresholding parameter $r=4$, for which the best results were obtained. For each value of $r$, the results are compared for four values of the offset parameter, $\delta=\{-0.4, -0.05, 0.05, 0.75\}$. 

Figure \ref{r4}(a) shows the percolation behavior of the algorithm for $r=4$ and various values of $\delta$. For $\delta=-0.05, 0.05$, and $0.75$, the percolation process exhibits a phase transition. That is, after some critical value of $\Lambda_0$, the algorithm maps almost every node in $G_1$ to some node in $G_2$. But for $\delta=-0.4$, the algorithm percolates almost linearly in the number of initial seeds, i.e. it fails to identify many nodes beyond the randomly given seeds. Hence, in this case, the algorithm requires a large value of $\Lambda_0$ to map a significantly large number of users which is not reasonable or practical. Figure \ref{r4}(b) shows the error rates for this scenario. We define the error rate as the ratio of incorrectly mapped nodes to the total number of mapped nodes excluding the seeds. The error rates seem to converge to small values, which means that, when the algorithm managed to percolate, it deanonymized the users correctly. 

Given that the reason for the failure of the algorithm, at  $\delta=-0.4$, is not the errors in mapping the users, but rather not being able to percolate, we tested the performance for $r=2$ and $r=3$. For $r=2$ the percolation process undergoes a phase transition for all the values of $\delta$. In this case the error rates were quite high and smaller values of $\delta$ result in even higher error rates as compared to larger values. In particular, for  $\delta=-0.4$, the error rate was more than $0.9$ even for a large number of initial seeds, which means that although the algorithm percolates, it deanonymizes only a small fraction of users correctly.

Figure \ref{r3}(a) shows the percolation behavior of the algorithm for $r=3$ and various values of $\delta$.  Note that, in contrast to the case $r=2$, the percolation for $\delta=-0.4$ shows similar behavior to that when $r=4$. Hence, even in this case, at lower negative values of $\delta$, the algorithm requires a large number of seeds to percolate efficiently. Figure \ref{r3}(b) shows the error rates of the algorithm for this case. Although the error rates are not too high, achieving them still requires a large number of seeds, especially for negative values of $\delta$.

On the basis of our results, we argue that, although $\delta<0$ is the threshold for the {\em exact} deanonymizability, the percolation algorithm fails significantly if we go somewhat lower than $0$. This is despite the fact that a large number of seed nodes was handed out to the algorithm as opposed to being learnt. The algorithm percolates for $r=2$ but makes large number of errors and hence only deanonymizes a small fraction of users. If $\delta$ decreases further, the error rate is expected to increase even more. For $r=4$, although the error rate is small, achieving this requires a large number of seeds, if $\delta$ is near or below zero, which imposes a limitation on the applicability of the algorithm in practice.

\subsection{Results for Real Network}
We consider a real world dataset and study its utility and privacy trade-off by varying the subsampling parameter $t$. We consider Facebook network \cite{snapnets}, containing $4039$ users and $88234$ edges. The average clustering coefficient is $0.6055$ and fraction of closed triangles is $0.2647$ which suggests strong community structure. We expect the community structure of the dataset to be resistant to edge perturbations. We describe the experimental methodology and results in the following subsections.
\subsubsection{Methodology} Our first aim is to study the effects of edge subsampling on the community structure. The original dataset is subsampled using the subsampling parameter $t=0.5,0.6,0.7,0.8$, and $0.9$. As there is no ground truth community labels for the network, we use the communities detected in the original network as our ground truth. To detect the communities, we use the freely available software Pajek \cite{batagelj1998pajek}, utilizing the Louvain modularity maximization method. We aim to measure the change in community structure as a function of $t$. We define the following parameters.
\begin{itemize}
	\vspace{-3pt}
	\item \textbf{Number of Communities}: A community is considered a true community only if it has at least $4$ vertices.
	\vspace{-13pt}
	\item \textbf{$\mathbf{(1-\boldsymbol{\epsilon})}$-Preservation}: We find the best match among the communities of the two networks using the Jaccard index, $J(A,B)=\frac{|A\cap B|}{|A\cup B|}$. Note that, higher the index, better the community is preserved. Once the best match has been found for all the communities, we define the $(1-\epsilon)$-Preservation as the number of communities with Jaccard index at least $(1-\epsilon)$. We consider $\epsilon \in \{0.1,0.15\}$, which ensures that the communities are preserved extremely well.
\end{itemize} 
Our second aim is to study the effect of sub-sampling on the anonymity of the dataset. For this purpose, we generate the auxiliary network by rewiring $30$ percent of the edges of the original network. This choice models an adversary with access to an auxiliary network which is highly correlated with the anonymized one. We then study the deanonymization results for percolation algorithm \cite{yartseva2013performance} for varying values of sub-sampling parameter $t$.
\subsubsection{Results} Table \ref{tablefb1} shows the results for the Facebook dataset. 
The number of communities is well preserved, the maximum cha-nge being $2$ for $t=0.6$. The size of the smallest community is preserved perfectly till $t=0.7$  whereas size of the largest community is always well preserved. This indicates that small communities tend to break into even smaller ones if we subsample too much. The most interesting part of the results is the $(1-\epsilon)$-Preservation. Subsampling upto $t=0.7$ preserves most of the communities to more than $90 \%$ of the members. Thinking less conservatively, even going upto $t=0.5$ preserves most of the communities to more than $85\% $ members. These results indicate that most of the community structure is well preserved even if we subsample to half the number of edges. Table \ref{tablefb2} shows the Jaccard indices for the five largest communities in the network. Note that the minimum size of a community in this scenario was $346$, and corresponds to the smallest community for $t=0.6$. As is evident from the table, most of these communities are preserved to over $95 \%$. The results seem to be an outcome of the already strong community structure in the original network. The results are motivating in the sense that preservation of community structure after edge perturbation depends on the strength of the communities in the original network.

\begin{table}[t]
	\caption{Number of communities, size of smallest and largest community and number of well preserved communities for varying values of sub-sampling parameter $t$.}\label{tablefb1}\vspace{-17pt}
\begin{center}
		\resizebox{\columnwidth}{!}{
	\begin{tabular}{ | c | c | c | c | c | c | c|}
		\hline
		\textbf{t} & \textbf{1} & \textbf{0.9} & \textbf{0.8} & \textbf{0.7} & \textbf{0.6} & \textbf{0.5} \\ \hline
		\textbf{No. of Communities} & 16 	& 16 	& 17 	& 17 	& 18 	& 17\\ \hline
		\textbf{Minimum Size} 		& 19 	& 19 	& 19 	& 19 	& 8 	& 6\\ \hline	
		\textbf{Maximum Size} 		& 548 	& 548 	& 548 	& 547 	& 547 	& 546\\		\hline
		\textbf{0.9-Preservation} 	& 16 	& 15 	& 13 	& 13 	& 11 	& 11\\		\hline
		\textbf{0.85-Preservation} 	& 16 	& 16 	& 13 	& 15 	& 13 	& 14\\		\hline
	\end{tabular}
}
\end{center}
\end{table}
\begin{table}[t]
			\caption{Jaccard index of the $5$ largest communities for varying values of sub-sampling parameter $t$.}\label{tablefb2}\vspace{-17pt}
	\begin{center}
		\begin{tabular}{ | c | c | c | c | c | c|}
			\hline
			\textbf{t$\to$} & \textbf{0.9} & \textbf{0.8} & \textbf{0.7} & \textbf{0.6} & \textbf{0.5} \\ \hline
			\textbf{$1^{st}$}  & 1 		& 1 		& 0.9982 & 0.9982 & 0.9964\\ \hline
			\textbf{$2^{nd}$}  & 1 		& 0.9816 	& 0.9834 & 0.9634 & 0.9757\\ \hline	
			\textbf{$3^{rd}$}  & 0.9794 & 0.9861 	& 0.8859 & 0.8812 & 0.8977\\ \hline
			\textbf{$4^{th}$}  & 0.9128 & 0.9777 	& 0.9596 & 0.9703 & 0.9477\\ \hline
			\textbf{$5^{th}$}  & 1 		& 0.9781 	& 0.9953 & 0.9802 & 0.9636\\ \hline
		\end{tabular}
	\end{center}
\end{table}

\begin{figure*}
	\center{
		\begin{tabular}{@{}c@{}c@{}}
			\includegraphics[scale=0.16]{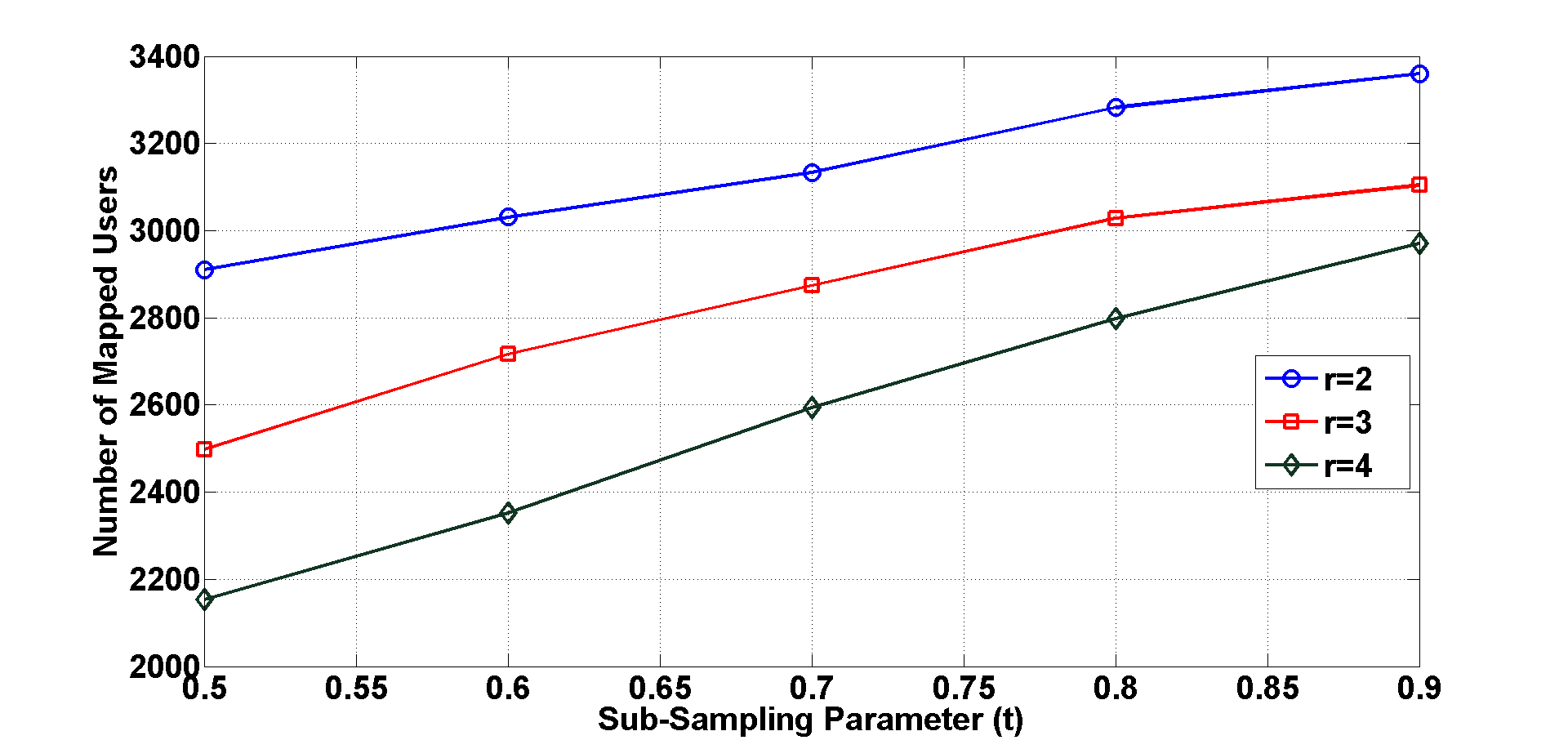}&\includegraphics[scale=0.16]{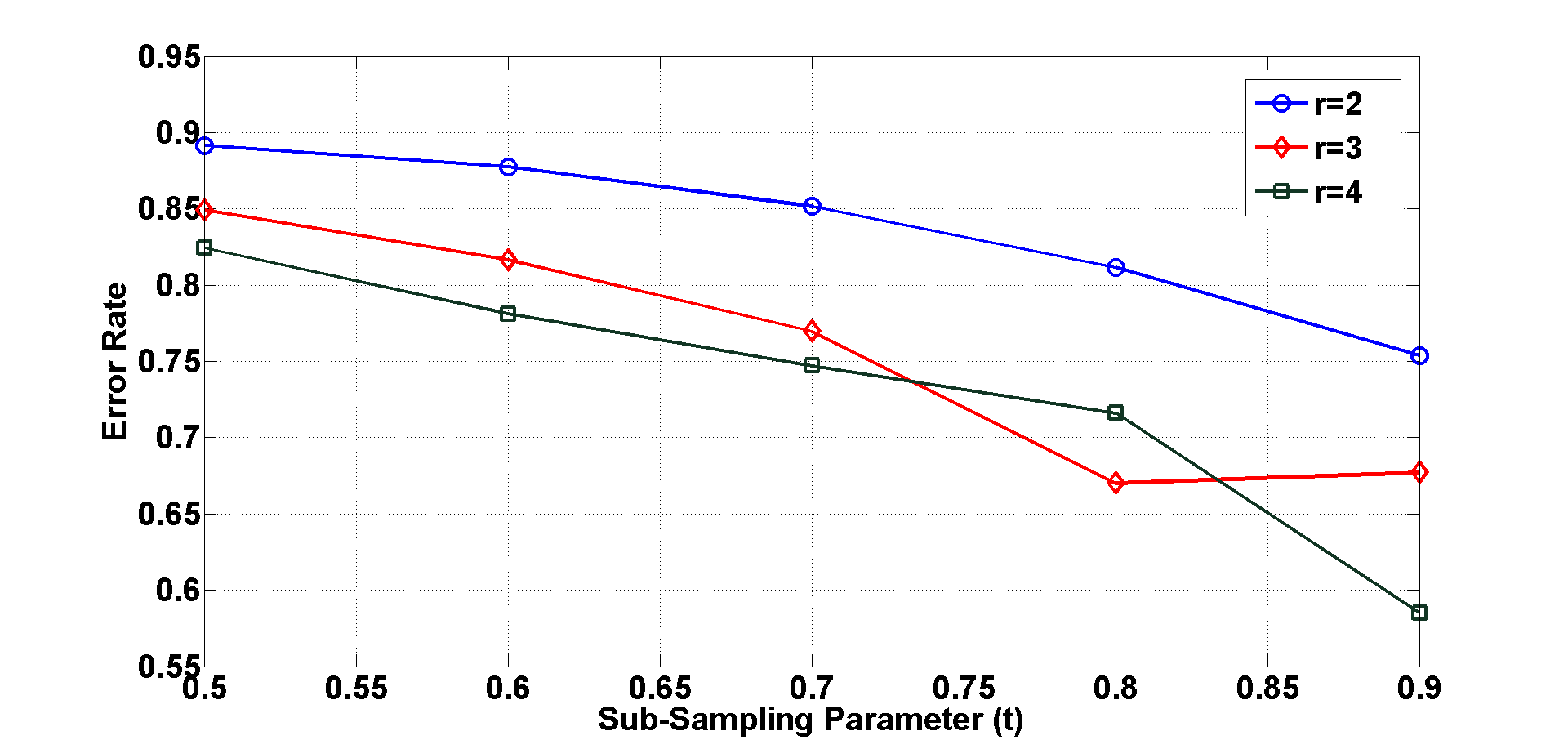}\\
			(a) & (b) \end{tabular}}\vspace{-11pt} \caption{{Deanonymization results using Percolation algorithm on for various values of percolation threshold $r$. (a) Total number of mapped users vs Subsampling parameter $t$, and (b) Error rate vs Subsampling parameter $t$
		}}\label{d1per}
	\end{figure*}
Figure \ref{d1per} shows the deanonymization results using the percolation algorithm with threshold $r=2,3,4$. We used $500$ number of random seeds. This selection was made keeping in mind that the algorithm should percolate while the number of seeds is practical as well. Also, as seen in Figure \ref{r4}, when near the threshold, the algorithm required around $500$ seeds to percolate. As evident in Figure \ref{d1per}(a), the number of mapped users increases with sub-sampling parameter $t$ for every value of $r$. A decreasing pattern is evident in Figure \ref{d1per}(b) for the error rate. The definition of error rate is the same as in subsection \ref{expsbm}. Note that for $t\leq 0.7$, the error rate is well above $75\%$ for all values of $r$. Hence subsampling this dataset to around $t=0.7$ preserves the anonymity of most of the users while still preserving most of the community structure. The best results seem to be obtained with $r=3$ as more users are mapped, compared to $r=4$ and the error rates seem to be similar. Even this choice maps around $35\%$ users when $t=0.9$. 

The results obtained indicate that the community structure is well preserved in the Facebook network at least upto $t=0.7$. Depending on the application, even going as low as $t=0.5$ preserves the community structure to a good extent. The deanonymization results also indicate that $t\leq 0.7$ ensures that most of the users remain anonymized. These results show that, depending upon the dataset, it is possible to preserve most of the community structure after edge perturbations while preserving the privacy. Most of the studies until now have missed this point. These results call for more dataset oriented research into the utility-privacy trade off.

%% file: conclusion.tex
\section{Conclusion}\label{sec:conclusion}
In this paper, we considered the problem of network de-anonymizability and established an information theoretic converse for the exact deanonymizability. This result applies to any deanonymization algorithm and hence provides a fundamental limit for this statistical estimation problem. This is qualitatively different from existing work in this area. We also improve the state of the art in achievability conditions, where significant effort has already been spent designing both efficient algorithms and information theoretically optimal methods. In particular, our converse and achievability bounds have the same parameter dependence. For exact deanonymization, the bounds match up to a factor of $2$. Our work supports the idea that the intersection graph of the auxiliary and sensitive networks plays a fundamental role in controlling the feasibility of deanonymization. This adds to existing evidence from \cite{yartseva2013performance}, where this intersection plays a crucial role in the analysis of percolation algorithm.

An important consequence of our result is that it is sometimes possible to prevent deanonymization while preserving other important structural information contained in the sensitive graph, particularly the community structure. The amount of ground truth information available to the public plays an important role in this trade-off.
 
The converse only establishes that a subset of the vertices cannot be deanonymized.
However, we make similar strong requirements in the community reconstruction problem: we require the community label of every vertex to be learnable.
The existance of a safe region under these very strict definitions of deanonymization and community recovery suggests that one might also exist if the definitions are simultaneously relaxed.
Additionally, simulations illustrate that deanonymization algorithms tend to fail drastically when correlation and edge density become too low. The failure conditions for these algorithms are not identical to the conditions of our converse, but they are closely related. 
Consequently we believe that it is possible to rigorously establish stronger impossibility results for deanonymization.